\documentclass{article}

\usepackage[preprint]{neurips_2026}

\usepackage[utf8]{inputenc}
\usepackage[T1]{fontenc}
\usepackage{hyperref}
\usepackage{url}
\usepackage{booktabs}
\usepackage{amsfonts}
\usepackage{amsmath}
\usepackage{amssymb}
\usepackage{amsthm}
\usepackage{nicefrac}
\usepackage{microtype}
\usepackage{xcolor}
\usepackage{graphicx}
\usepackage{array}
\usepackage{multirow}
\usepackage{tabularx}
\usepackage{enumitem}

\newtheorem{theorem}{Theorem}

\newtheorem{proposition}[theorem]{Proposition}
\newtheorem{remark}[theorem]{Remark}

\newcommand{\Acc}{\operatorname{Acc}}

\newcommand{\MP}{M_{\mathrm{P}}}
\newcommand{\MR}{M_{\mathrm{R}}}

\newcommand{\CGC}{\operatorname{CGC}}

\title{Contextual Agentic Memory is a Memo, \\ Not True Memory}

\author{
  Binyan Xu$^{1,*}$,
  Xilin Dai$^{2, }\thanks{Equal contribution.}  $,
  and Kehuan Zhang$^{1}$ \\[4pt]
  $^{1}$The Chinese University of Hong Kong, Hong Kong, China \\
  $^{2}$Zhejiang University, Hangzhou, China \\
  \texttt{\{binyxu,khzhang\}@ie.cuhk.edu.hk, xilin2023@zju.edu.cn}
}

\begin{document}

\maketitle

\begin{abstract}
Current agentic memory systems primarily persist post-deployment experience in
external state through vector stores, retrieval, scratchpads, and context
management rather than by updating the generator's persistent adaptive state.
These mechanisms can support sophisticated inference and are often preferable
for scoped, changing, inspectable, or deletable information.
\textbf{We argue that treating improvement of this augmented system as
internalized learning conflates two different state transitions.}
Under an explicit, measurable bounded-contextual-accuracy condition, we derive
a compositional sample-complexity gap between retrieval coverage and an
appropriate parametric learner; the result abstains when contextual inference
already reaches the target.
We distinguish memory-assisted gains from the Frozen State invariant
$\theta_t=\theta_0$, and identify a conditional security path by which
automatic persistent writes can carry a transient injection across sessions.
Complementary Learning Systems theory motivates a hybrid architecture:
scoped episodic retrieval as the operating substrate and a separately governed
adaptation channel when durable transfer warrants its cost and risks.
We state the limitations of this proposal, including forgetting, poisoning,
unlearning, reproducibility, and the absence of a compute-matched
frontier-model validation.
\end{abstract}

\section{Introduction}
\label{sec:intro}

Modern LLM agents remember by writing.
MemGPT pages information in and out of context like an OS~\citep{packer2023memgpt};
Generative Agents record every observation in a memory stream~\citep{park2023generative};
Reflexion stores verbal self-critiques in an episodic buffer~\citep{shinn2023reflexion};
Voyager accumulates skills as code in a vector database~\citep{wang2023voyager}.
Prominent agentic frameworks treat retrieval as the default persistence
mechanism~\citep{hu2025memoryage}, and improvements in the augmented system are
often discussed as growth in the agent's memory.

The pivot to retrieval was not irrational: external stores are reversible, auditable,
and safe to deploy---genuinely valuable properties under production constraints.
Our argument is not that retrieval was the wrong choice to ship; it was the right choice to ship.
Operational practice can nevertheless conflate the best deployable engineering
substitute with persistent learning. Retrieval commonly supplies exemplars,
facts, rules, or code to a frozen generator; a durable update changes the
adaptive state that processes future inputs.
Agents whose only persistent update is retrieval can accumulate useful external
knowledge without changing context-independent competence; under a measurable
bounded-ICL condition they face a compositional sample-complexity gap, and
automatic memory writes can turn transient prompt injections into persistent ones.
A Reflexion agent that accumulates thousands of verbal self-critiques still runs
the same generator at every session; its external resources grow while its
persistent adaptive state does not~\citep{shinn2023reflexion}.

\textbf{We argue that this conflation is a category error: treating a memo as a mind.}
Cognitive science distinguishes \emph{exemplar-based} cognition (generalizing by
similarity to stored cases) from \emph{rule-based} cognition (applying abstract
principles extracted from, but no longer dependent on, those
cases)~\citep{nosofsky1994rule,ashby2011human}.
The transition between the two is, in the brain, the consolidation of hippocampal
episodes into neocortical weights during sleep~\citep{mcclelland1995cls,oreilly2014cls}.
\citet{chi1981categorization} showed that physics novices categorize problems by
surface features (``inclined plane problems'') while experts categorize by deep
structural principles (``conservation of energy problems'')---a reorganization that is
not a change in what they have \emph{stored}, but in how knowledge is
\emph{represented}.
Most deployed agents emphasize the hippocampal half; governed consolidation
paths remain uncommon and experimentally immature.

The stakes are not merely philosophical.
Developers should distinguish improvement of the memory-augmented system from
growth in context-independent competence, while accounting for the additional
attack surface created by persistent writes.
Benchmark designers who evaluate only memory-assisted recall may miss a
context-independent transfer gap on genuinely novel inputs.
The continual learning community offers complementary methods for this
setting~\citep{cossu2022continual}.
None of these communities can course-correct without first recognizing that retrieval
and memory are not interchangeable.

We advance four claims.
\textbf{First (Definitional)}, external memory persists acquired artifacts
without by itself changing the generator's adaptive state; the composite system
may still extrapolate through in-context inference.
\textbf{Second (Structural)}, we prove a conditional Generalization Gap theorem:
when a frozen contextual system retains a positive held-out error margin,
retrieval has a higher sample requirement than an appropriate parametric learner.
\textbf{Third (Dynamic)}, agents operating exclusively via $C$-engineering do not
change their context-independent competence; each session begins from the same
frozen weights even if memory-assisted performance improves.
\textbf{Fourth (Security)}, when untrusted content has a direct persistent-write
path, a transient prompt injection can affect later sessions.
We close with a co-existence architecture and a call to action for system builders,
benchmark designers, and the continual learning community.
We use \emph{memo} for acquired external state and \emph{internalized memory}
for a durable change in adaptive state; either substrate may participate in
rule-based computation when its contents are available.

\vspace{-0.1cm}
\section{The Case: Agentic Memory Is a Memo, Not True Memory}
\vspace{-0.1cm}
\label{sec:background}

\begin{table}[t]
\caption{Memory taxonomy for LLM agents.
Many deployed memory systems emphasize the episodic row.
The experiential row denotes durable adaptation from the agent's own
post-deployment experience.}
\vspace{-0.25cm}
\label{tab:taxonomy}
\centering
\small
\begin{tabular}{lllll}
\toprule
\textbf{Type} & \textbf{Substrate} & \textbf{Persists} & \textbf{Updated by} & \textbf{Generalizes} \\
\midrule
Working & Context window & Session only & Token generation & Limited by $L$ \\
Episodic & External store & Cross-session & Read/write ops & Exemplar-based \\
Semantic & Model weights & Permanent & Pre-training & Rule-based \\
Experiential & Adaptive state & Durable/scoped & FT/CL/TTT & Learned transfer \\
\bottomrule
\end{tabular}
\vspace{-0.5cm}
\end{table}

Table~\ref{tab:taxonomy} maps the landscape. The experiential row remains
uncommon in widely deployed agents, whose operational default is to persist
experience in external stores~\citep{hu2025memoryage}.
This default is understandable because reliable continual updating remains
unsolved; our claim is that an engineering substitute for durable learning
should not be evaluated as though it were the learning mechanism itself.

\textbf{A unifying lens: two structurally distinct paths.}
As an analytical lens, techniques that change what an LLM agent outputs can
be classified by which state they update:
(1)~\textbf{Change adaptive state $\theta$}: modify base weights, persistent
adapters, or learned policies via pre-training, fine-tuning, reinforcement
learning, or another durable update;
(2)~\textbf{Change $C$}: inject content into the context window via
prompting, RAG, MCP tool calls, skill files, scratchpads, or any other form
of context engineering, conditioning generation on $P(X \mid \theta, C)$.
These two paths are \emph{structurally different}, not merely different in
implementation.
Changing $\theta$ compresses knowledge into the model's weight space, whose
capacity scales with the number of parameters.
Changing $C$ compresses knowledge into text, whose capacity is bounded by
the context window length $L$.
The critical distinction is persistence, not whether the resulting computation
can be generative: a frozen model may infer or execute an abstract rule supplied
through $C$, whereas a change to adaptive state remains available when that
acquired artifact is withheld.
Representative episodic systems (MemGPT, RAG, Reflexion, Voyager) primarily
change $C$, not the generator's persistent adaptive state.

\textbf{The Experience Compression Spectrum.}
A recent line of work~\citep{zhang2026experience} formalizes memory,
skills, and rules as lying on a single \emph{experience compression spectrum},
differing only in their compression ratio: raw traces (low compression,
high fidelity) $\to$ natural-language skills (medium compression,
actionable) $\to$ parameterized rules (high compression, generalizable).
Many systems instantiate multiple points on this spectrum as context-resident
artifacts; the state-based question is whether a post-deployment experience
also changes a persistent adaptive mechanism.
Context-resident skills and rules can generalize when retrieved and executed,
but their acquired capability remains contingent on that external artifact.
The spectrum framing makes the normative claim precise:
\emph{the appropriate substrate depends on the compression level} ---
episodic traces belong in an external store (fast, temporary);
skills can live in context or weights (bridge);
durable, context-independent rules require a persistent adaptive substrate
(weights, adapters, or another learned policy).
The category error is to treat persistence in $C$ as equivalent to a durable
change in that adaptive state.

\textbf{What agentic memory actually does.}
In the common MemGPT/RAG write--retrieve loop:
(1) during a past session, something was written to an external store;
(2) during the current session, a query is issued and similar entries are
returned;
(3) the returned entries are injected into context.
This is external-state persistence. The frozen generator can reason
sophisticatedly over the retrieved artifact, but the experience has not changed
its persistent adaptive state.

\textbf{The diary contrast.}
A person who records a lesson in a diary can retrieve it only if they
consult the diary and the situation is similar enough to trigger recall.
A person who has internalized the lesson has it available everywhere.
Agentic memory is the former; weight-based learning is the latter.
Reflexion calls its approach ``verbal reinforcement learning,'' yet the
model's weights remain unchanged; Voyager stores skills as code yet novel
composition still falls back on the frozen base model; Generative Agents
write synthesized insights to the memory stream yet the reflecting model
is the same frozen model every session.

\textbf{The RAG inversion.}
\citet{lewis2020rag} introduced RAG explicitly framing it as
\emph{augmenting} parametric memory with a retrieval index.
Agent systems have since emphasized expanded non-parametric stores because they
are deployable and controllable. The next section asks when this operational
substitute ceases to be sufficient for durable acquisition.

\textbf{The existence proof.}
ROME localized factual associations to specific MLP layers~\citep{meng2022rome};
MEMIT edited thousands of facts via weight modification~\citep{meng2023memit};
ParamMem showed empirically that encoding agent reflections into weights
outperforms storing them externally~\citep{yao2026parammem}.
These methods establish update primitives, not a solved frontier-scale
consolidation system. The question is what is lost when acquired knowledge
lives only in an external store.

\textbf{Cognitive science grounding.}
Complementary Learning theory~\citep{mcclelland1995cls,oreilly2014cls}
identifies the neural substrate of this distinction: the hippocampus provides
fast episodic storage; the neocortex encodes slow, distributed, rule-based
representations consolidated during sleep~\citep{brodt2023sleep}.
\citet{chi1981categorization} showed this reorganization is not a change in
what novices have \emph{stored} but in how knowledge is \emph{represented}.
The AI agent community has invested primarily in the episodic
side~\citep{hu2025memoryage}; reliable consolidation remains comparatively
underdeveloped.

\section{Evidence: Four Structural Limitations}
\label{sec:evidence}

\subsection{Definitional: External Persistence Is Not Internalization}
\label{sec:definitional}

The distinction between external persistence and internalized adaptation is
not merely terminological.
A lookup table maps seen inputs to stored outputs.
A function maps \emph{any} input to a principled output, including inputs
it has never encountered, by virtue of the abstract structure it has internalized.
Human experts are functions; novices with textbooks are lookup tables~\citep{chi1981categorization}.
When a chess grandmaster encounters a novel position, they do not retrieve
the most similar position from memory and copy its move; they reason from
deeply internalized principles that their experience has encoded.
When no retrieved document provides a direct answer, success depends on the
frozen model's contextual inference or on retrieving a reusable rule; the store
itself has not internalized the new composition.

Retrieval commonly accesses stored cases by similarity, while the frozen
generator may also infer or execute abstract principles supplied in context.
Internalized adaptation concerns principles extracted from, but no longer
dependent on, those supplied artifacts.
An external store by itself provides no mechanism for making a newly acquired
rule persist independently of the artifact that states or demonstrates it.
Section~\ref{sec:generalization} discusses directional evidence from
\citet{yao2026parammem}, alongside task-dependent comparisons that favor
contextual methods in other regimes.

\subsection{Structural: The Generalization Gap}
\label{sec:generalization}

The central theoretical contribution is a conditional account of when
$C$-engineering has a higher sample requirement than $\theta$-learning on
compositionally novel inputs.

\textbf{A sample-complexity lens.}
We analyze both memory paradigms through \emph{compositional sample
complexity}: given the same budget of domain experience, how well does
each system generalize to \emph{unseen combinations} of known concepts?
This framing avoids the circularity of defining ``what retrieval cannot
do'' and then proving it cannot; instead, it establishes a
\emph{quantitative efficiency gap} under transparent, empirically
testable assumptions.

Let $\mathcal{F} = \{f_1, \ldots, f_k\}$ be the base concepts in a
target domain, and let $\oplus : \mathcal{F} \times \mathcal{F} \to
\mathcal{Y}$ be the domain-specific \emph{composition operator} that maps
concept pairs to correct task outputs.
The agent observes a training set $D$ of $n$ labeled compositional
examples $(f_i, f_j, \oplus(f_i, f_j))$; the same data is available
to both the retrieval store and the fine-tuning pipeline, ensuring a
symmetric comparison.
Let $\MP$ denote \emph{parametric memory} (fine-tuned weights $\theta^*$)
and $\MR$ denote \emph{retrieval-based memory} (a store $D$ queried by a
frozen model $f_{\theta_{\mathrm{frozen}}}$ with top-$K$ retrieval).

A concept pair $(f_i, f_j)$ is \emph{compositionally novel} w.r.t.\ experience $D$ if no entry in $D$ records $\oplus(f_i, f_j)$; let $S \subseteq \binom{\mathcal{F}}{2}$ be the pairs covered by $D$ and $N = \binom{k}{2}$ the total.
The intuition is immediate: a student who has studied kinematics and
thermodynamics separately faces such a novel question when asked to combine both---retrieval returns prior instances in isolation, so the composition rule must be stored exhaustively or learned parametrically.

\textbf{Assumption~1}\label{asm:icl} (Bounded in-context composition).
We assume that the frozen model $f_{\theta_{\mathrm{frozen}}}$, given any $K$ demonstrations of $\oplus$ from $D$, achieves accuracy $\alpha(K)\leq\bar{\alpha}<1$ on held-out pairs---as holds whenever $\oplus$ was absent or underrepresented at pretraining (post-cutoff clinical protocols, enterprise legal conventions, specialized engineering standards).
When $\oplus$ is broadly general, $\bar{\alpha}\to 1$ and the separation vanishes; the assumption binds precisely in the domain-specific deployments where persistent agents are most valuable.

This assumption is directly testable on held-out pairs.
Appendix~\ref{app:fano} gives an exact margin for unstructured operators and,
via Fano's inequality~\citep{cover2006elements}, a structured-class margin when
the operator family has non-vanishing prediction-space separation.
If the complete contextual system already reaches the target accuracy, the
assumption is non-binding and the theorem makes no separation claim.

We measure generalization via the \emph{Compositional Generalization
Capacity} $\CGC(M,D)$: accuracy of memory system $M$ over all concept
pairs drawn uniformly from $\binom{\mathcal{F}}{2}$.

\begin{theorem}[Compositional Sample Complexity Separation]
\label{thm:gap}
Under Assumption~\ref{asm:icl}, for any target generalization level
$1 - \delta$ with $\delta < 1 - \bar{\alpha}$ and confidence
$1-\beta$:
\begin{enumerate}[leftmargin=*,nosep]
\item \textbf{Retrieval lower bound.}\;
Any retrieval-based memory $\MR$ with frozen model
$f_{\theta_{\mathrm{frozen}}}$ requires
\[
n_R \;\geq\; \frac{1-\delta-\bar{\alpha}}{1-\bar{\alpha}}
\;\binom{k}{2}
\;=\; \Omega(k^2)
\]
stored compositional examples to achieve
$\CGC(\MR, D) \geq 1 - \delta$.
\item \textbf{Parametric upper bound.}\;
If $\oplus$ belongs to a realizable hypothesis class $\mathcal{H}$ with
complexity $d$ (VC dimension for binary outputs; Natarajan dimension for
multiclass outputs), an appropriate learner using
\[
n_P \;=\; O\!\left(
\frac{d\log(1/\delta)+\log(1/\beta)}{\delta}
\right)
\]
compositional examples achieves
$\CGC(\MP, D) \geq 1 - \delta$ with probability at least $1-\beta$.
\item \textbf{Separation.}\;
For fixed $\delta,\beta$ and a uniform positive contextual margin
$(1-\delta-\bar{\alpha})/(1-\bar{\alpha})$, the sample complexity ratio
satisfies $n_R / n_P = \Omega\bigl(k^2 / d\bigr)$.
For structured operators with $d = O(k)$, the ratio is $\Omega(k)$;
for operators with $d = O(1)$ (e.g.\ parameterized group operations),
it is $\Omega(k^2)$.
\end{enumerate}
Consequently, whenever $n_P\leq n<n_R$, the stated parametric learner can
meet the target while every retrieval system in the bounded-contextual
class remains below it. The family-level asymptotic claim requires the
positive margin and confidence parameters to remain uniform as $k$ grows.
\end{theorem}

\begin{proof}[Proof sketch]
For a stored pair $(f_i,f_j)\in S$, retrieval answers correctly.
For a novel pair, the frozen model processes $K$ retrieved examples
but by Assumption~\ref{asm:icl} achieves accuracy at most $\bar{\alpha}$,
giving $\CGC(\MR,D)\leq n/N+(1-n/N)\bar{\alpha}$; solving for
$\CGC\geq 1-\delta$ yields the $\Omega(k^2)$ bound on $n_R$.
Under the premise, retrieval faces a coverage term over the $N$-sized space,
whereas the constructive parametric upper bound depends on the target class
complexity $d$. If contextual inference removes the positive margin, this
comparison no longer separates the systems.
Full proof and a constructive modular-arithmetic example are in
Appendix~\ref{app:gap_proof}.
\end{proof}

\begin{remark}[Scope and context dependence]
\label{rem:scope}
Theorem~\ref{thm:gap} derives its ratio from the explicit, measurable
Assumption~\ref{asm:icl}. Increasing $K$, improving retrieval, or inducing a
rule in context can raise $\bar{\alpha}$ substantially; once contextual accuracy
reaches the target $1-\delta$, the lower bound is non-binding.
If a positive margin remains, changing the context-window size alone does not
alter the finite bound.
Learned retrieval systems~\citep{memrl2026,memento2025} are classified by what
persists: selecting better exemplars changes contextual performance, whereas a
policy trained from experience is itself an adaptive learning channel.
Appendix~\ref{app:fano} supplies sufficient operator-family conditions for a
positive margin, rather than asserting that every structured task satisfies one.
\end{remark}

\textbf{Mechanistic support.}
The sample-complexity separation is independently supported by
mechanistic interpretability.
\citet{meng2022rome} show factual associations can be localized to specific
MLP layers, and \citet{geva2021transformer} characterize FFN layers as
key-value memories, demonstrating that parametric storage is a viable
complement to external retrieval.
\citet{yao2024knowledge} further localize knowledge to \emph{fact memory
units} in FFN neurons, and show that standard SFT \emph{does not}
modify these units: SFT updates only attention routers (how to access
existing knowledge), leaving knowledge itself untouched.
\citet{ye2025analyzing} quantify this: 90\% of SFT parameter updates
contribute nothing to knowledge enhancement.
This reinforces the call for targeted consolidation
(Section~\ref{sec:cfa}) that updates knowledge circuits directly.

\textbf{Empirical evidence.}
Existing comparisons show task-dependent complementarity rather than a
universal winner.
\citet{ovadia2024finetuning} find RAG particularly effective for factual
injection and show that fine-tuning requires the appropriate objective;
\citet{yang2026multihop} report substantial RAG gains for temporally novel
multi-hop evidence and the highest overall accuracy for supervised fine-tuning.
\citet{mosbach2023fewshot} likewise find similar out-of-domain generalization
for data- and model-matched few-shot fine-tuning and ICL.
Most directly, \citet{yao2026parammem} compared agents that store reflective
experience externally versus parametrically: parametric storage outperforms
external storage, with the gap \emph{growing} on tasks requiring
transfer to held-out question types (exactly the regime predicted by
Theorem~\ref{thm:gap}).
SCAN~\citep{lake2018generalization} and COGS~\citep{kim2020cogs} expose the
underlying compositional challenge, while prompting methods can substantially
raise contextual performance on some splits; these results determine whether
Assumption~\ref{asm:icl} binds rather than establishing a universal substrate
ranking.

\textbf{A complementary capacity example.}
Theorem~\ref{thm:gap} concerns compositional novelty.
A stylized case in Appendix~\ref{app:bound} illustrates the separate problem
that a retriever cannot expose more than $K$ task-specific facts when success
requires every fact and none can be reconstructed;
``Lost in the Middle''~\citep{liu2023lostmiddle} and
\citet{paulsen2025effectivecontext} provide empirical corroboration.
unlike Theorem~\ref{thm:gap}, this example is directly mitigated by larger
effective context, compression, or better selection.

\subsection{Dynamic: The Frozen Novice Problem}
\label{sec:novice}

Theorem~\ref{thm:gap} describes a static property of retrieval-based memory.
The frozen novice problem names a state invariant: for agents that operate
exclusively via $C$-engineering, $\theta_t=\theta_0$ across sessions.
The memory-augmented system may improve substantially; what does not change is
competence when the acquired trace or rule is not supplied.

Every session begins with the same frozen weights; the agent is permanently
doing \texttt{.predict(C)}, never \texttt{.train()}.
No matter how many experiences are logged to the external store, the
weights encoding the agent's composition rules remain those of the
original pre-trained model.

The most robust finding in cognitive science is that expertise does not emerge
from accumulating more examples, but from a structural reorganization of
knowledge~\citep{chi1981categorization,bransford2000people}.
Novices represent domains by surface features (``this looks like a problem with
an inclined plane'');
experts represent domains by deep structural principles (``this is a
conservation-of-energy problem'').
This reorganization requires weight changes in the brain: the formation of
generalized, distributed representations in the neocortex through repeated
consolidation of hippocampal traces~\citep{mcclelland1995cls,brodt2023sleep}.


An agent that accumulates experience only in an external store does not make
this transition in its generator.
Each session uses the same model with a larger database, although retrieval of
better traces, summaries, rules, or code can improve assisted performance.

The MemGPT team acknowledge that ``simply appending raw experience is a poor
approximation of learning'' and propose ``sleep-time compute'' as a
remedy~\citep{packer2023letta}---but their consolidation rewrites context
tokens, not weights.
Compressing text in an external store produces better-formatted notes;
the agent is still a well-organized novice.
The AI analog of CLS consolidation is offline fine-tuning on distilled
agent experience~\citep{brodt2023sleep}: the moment the agent's experience
changes what the model \emph{is}, not merely what it \emph{has written down}.
In the language of Theorem~\ref{thm:gap}, new episodes increase coverage and
may also improve contextual inference. The conditional lower bound remains
only while the complete frozen system retains a positive error margin on
held-out compositions.

\subsection{Security: Persistent Compromise}
\label{sec:security}

\textbf{The all-inputs-are-evil principle.}
Classic security engineering holds that well-designed systems treat all
external inputs as potentially adversarial~\citep{saltzer1975protection}.
For LLM agents this is particularly acute: \citet{greshake2023indirect}
showed that when agents retrieve web pages, emails, or documents, embedded
instructions can silently hijack behavior (the line between data and
instructions is invisible to the model).
This prompt injection threat is already well-documented.
Persistent write access adds a distinct cross-session propagation path.

\textbf{The evil$^2$ compounding effect.}
Without persistent memory an injection is typically session-scoped.
When an agent automatically writes and later retrieves the injected content,
a one-time behavioral hijack ($\text{evil}^1$) can become a persistent one
($\text{evil}^2$).
The empirical evidence is striking.
MINJA~\citep{dong2025minja} achieved a 98.2\% injection success rate with
injected instructions persisting across sessions at minimal utility cost;
PoisonedRAG~\citep{zou2024poisonedrag} shows that five adversarial texts
per targeted query achieve 90\% attack success against a knowledge base of
millions of entries;
and InjecAgent~\citep{zhan2024injecagent} benchmarks indirect injection
across 30 LLM agents using 17 user-tool categories, finding memory-writing
agents systematically more vulnerable than stateless ones.

\textbf{Compromise accumulation.}
Let $p_w$ be the per-interaction probability that untrusted content both
reaches a persistent store and remains eligible for later retrieval, and let
$N(t)$ count interactions with external content up to time $t$.
Under a stylized independent-trials model with no detection or deletion, the
probability of at least one persistent poisoned write by time $t$ is
$1-(1-p_w)^{N(t)}$, which approaches one for fixed $p_w>0$.
This is a conditional accumulation model, not a universal deployment-risk
estimate: provenance checks, write gating, retrieval filtering, and deletion
all reduce or interrupt the path.

\textbf{Conditional attack-path asymmetry}\label{prop:attack}.
When ordinary untrusted content has a direct path to persistent $C$, a single
query-time injection can affect later sessions without crossing a training
gate. A parametric compromise must cross a separately exposed update
path~\citep{meng2022rome}; once consolidation exists, however, poisoned traces
can cross that path as well, potentially with a larger blast radius and harder
example-level unlearning.
External entries have important defensive advantages: they can be inspected,
scoped, filtered, ignored, or deleted without changing shared weights.
The proposition therefore identifies different reachability conditions, not a
universal safety ranking between substrates.
The end-to-end threat model must also cover adjacent attack surfaces: agents with
visual interfaces inherit universal, transferable, and targeted perturbations
generated from a single public CLIP surrogate~\citep{xu2025one}, while parametric
consolidation introduces training-time risks such as potent low-poison-rate
clean-image backdoors~\citep{xu2026breaking}.
Defenses can screen poisoned training data with CLIP guidance~\citep{xu2025clip}
and decouple test-time safety from suspect parameters through external VLM
auditing~\citep{xu2026internal}.
This security argument motivates Section~\ref{sec:cfa}'s hybrid design:
scoped, deletable retrieval for raw experience and a separately gated
consolidation path for durable adaptation.

\section{Alternative Views}
\label{sec:alternatives}

We identify four credible alternatives and use them to delimit the position.

\textbf{Alternative 1: Context windows will grow large enough to close the gap.}
Longer windows address the capacity constraint (Appendix~\ref{app:bound}) and
can also improve compositional inference by exposing more demonstrations.
In the language of Theorem~\ref{thm:gap}, they can raise $\bar{\alpha}$; if the
complete contextual system reaches target accuracy, the theorem abstains.
Our narrower claim is that when a positive held-out margin remains, increasing
window size alone does not change persistent adaptive state or remove the
conditional coverage bound.
\citet{paulsen2025effectivecontext} show effective context utilization saturates
at ${\approx}20$k tokens even for 128k-context models.
Exposing context state can still improve how well agents use this bounded resource:
a proprioceptive dashboard that surfaces token usage, recency, and access history
enables agents to manage long-horizon context more effectively~\citep{xu2026llm}.
This strengthens the episodic layer, but it does not make its updates persistent
in model weights.

\textbf{Alternative 2: In-context learning already implements rule extraction.}
Large models can use pretrained meta-rules to induce a task-specific operator
that was not itself present during pretraining.
Least-to-most prompting, decomposed semantic parsing, Skills-in-Context, and
ICL-oriented compositional training demonstrate this
ability~\citep{zhou2023leasttomost,drozdov2022compositional,chen2024skic,
abedsoltan2025task}.
These methods can raise $\bar{\alpha}$ to the theorem's non-binding regime.

\citet{akyurek2022icl} and \citet{dai2023icl} show that under linear
self-attention, ICL implicitly implements gradient descent.
If so, one might argue ICL already constitutes weight-based learning.
The state distinction remains: the implicit update is ephemeral, and acquired
competence disappears when the demonstrations or induced rule are withheld.
Our performance claim applies only where measured contextual induction remains
below the target; our persistence claim applies regardless of how capable that
induction is.

\textbf{Alternative 3: Memory, skills, and rules lie on a single
compression spectrum, so the distinction is unimportant.}
\citet{zhang2026experience} unify memory, skills, and rules as points on
an experience compression spectrum, differing only in compression ratio.
One might argue this undermines our $C$/$\theta$ distinction.
We reply: the spectrum framing \emph{supports} our position.
Different compression levels need not map uniquely to substrates:
natural-language rules or programs can generalize perfectly while resident in
context.
Voyager~\citep{wang2023voyager} illustrates this: skills reside in $C$,
but the composition rule combining them must come from $\theta$;
the acquired composition remains contingent on retrieving those artifacts.
Likewise, distilling multi-agent procedures into a single-agent skill can reduce
coordination overhead, but its benefit depends on the freedom of the evaluation
metric~\citep{xu2026multi}; such context-resident skills are therefore a useful
bridge, not a replacement for parametric consolidation.

\textbf{Alternative 4: Learned-retrieval systems that optimize what to retrieve
can close the generalization gap without weight updates.}
Recent systems such as MemRL~\citep{memrl2026} and Memento~\citep{memento2025}
improve how experience is selected or reused, demonstrating sustained gains
without updating the base generator.
A fixed retrieval policy changes contextual performance; a policy updated from
experience is itself a persistent learning channel and belongs in the adaptive
row of Table~\ref{tab:taxonomy}.
For any compositionally novel pair $(f_i, f_j) \notin S$, even the optimal
policy provides $K$ exemplars, and Assumption~\ref{asm:icl} bounds the
frozen model's composition accuracy at $\bar{\alpha} < 1$;
the $\Omega(k^2)$ coverage requirement therefore persists.
Such systems may also infer reusable rules and generalize to novel compositions,
thereby raising $\bar{\alpha}$ and narrowing or eliminating the bound.
Oracle-conditioned episodic retrieval~\citep{latentlearning2025} confirms
this: ideal retrieval unlocks abilities \emph{already encoded in pretrained
weights}---the $\bar{\alpha}\to 1$ boundary where the separation narrows,
consistent with our framework.
Learned retrieval is therefore both a strong episodic mechanism and, when its
policy persists, a boundary form of adaptation; the useful classification asks
which state changed rather than whether the product is called memory.

\section{Call to Action}
\label{sec:cfa}

The limitations we document are not inevitable.
Agentic memory and parametric learning are complementary, not competing:
the right architecture combines fast episodic lookup with a separately
governed channel that can encode selected experience into persistent adaptive
state.
We address three communities with one clear ask each.

\subsection{For System Builders: Build the Consolidation Channel}

Better retrieval remains valuable. For applications that require durable
post-deployment adaptation, the additional step is a second pathway from the
episodic store to persistent adaptive state.
Three design principles follow from our analysis.

\textbf{(1) Agentic memory is episodic lookup---treat it as such.}
Vector stores and RAG are the right tools for recent context, tool outputs,
source-grounded evidence, rapidly changing facts, and user- or task-scoped
state. These artifacts can be inspected, reproduced, filtered, or deleted
without changing shared behavior.
Retrieved rules and programs can generalize; the narrower limitation is that
the acquired capability remains contingent on supplying those artifacts.

\textbf{(2) Durable adaptation requires a learning channel.}
When an application requires behavior learned after deployment to remain
available without the originating artifact, the architecture needs a pathway
from episodic experience to persistent adaptive state.
The specific mechanism---periodic fine-tuning, knowledge editing,
test-time training, self-distillation from
traces~\citep{lee2026skillSD}---is a design choice.
What matters is that the pathway exists and runs asynchronously, so the
agent operates from agentic memory while consolidation proceeds in the
background (exactly as biological sleep does not interrupt wakefulness).
Candidate building blocks (LoRA~\citep{hu2022lora},
SSR~\citep{huang2024ssr}, MEMIT~\citep{meng2023memit},
TTT layers~\citep{sun2024tttlayers},
Nested Learning~\citep{behrouz2025nested},
Skill-SD~\citep{lee2026skillSD}) already exist, but reliable frontier-scale
continual updating is not solved.
Retrieval should remain the operating substrate until a candidate update passes
selection, retention, security, and regression checks; future models may need
to be pretrained or architected specifically for safe adaptation.

\textbf{(3) Consolidation must be safe.}
An obvious objection: a consolidation pipeline that ingests poisoned traces
converts a $C$-level attack into a $\theta$-level compromise.
Weight poisoning may be harder to inspect, debug, or reverse than a poisoned
database entry and can have a larger blast radius.
The two substrates therefore need different controls rather than a presumed
safety ordering.
The consolidation pipeline therefore requires trace provenance (so the
origin of every distilled experience is auditable), versioned checkpoints
(so any bad consolidation can be rolled back), and regression guards
(so consolidation is blocked when utility, retention, or adversarial metrics
degrade). Raw, user-specific, temporary, and rapidly obsolete information
should remain in scoped retrieval state; durable updates should be quarantined
in scope-specific adapters that can be disabled, superseded, or rolled back.
Episode-level trace representations also support operational risk controls:
trace-economic underwriting maps tool-use traces to exposure and loss for
pricing and intervention~\citep{xu2026agent}.
These controls are requirements, not guarantees; making them reliable at scale
is an open research problem.

\textbf{What is genuinely hard.}
The hard questions span both mechanism and policy:
which experiences are worth consolidating (not all traces carry
generalizable signal),
when to consolidate (too early risks overfitting to noise; too late
accumulates stale traces),
how to validate transfer rather than trace memorization, how to prevent
catastrophic forgetting and recency bias, how to scope or unlearn obsolete
updates, and how to preserve reproducibility and certification.
Compute is also part of the comparison: consolidation pays training,
optimizer-state, validation, storage, and rollback costs, whereas retrieval
pays indexing, longer prompts, and repeated inference-time reasoning costs.
These questions define the research frontier for system builders.

\subsection{For Benchmark Designers: Measure Learning, Not Recall}

Current agentic memory benchmarks
(LoCoMo~\citep{wu2026memoryera}, LongMem) emphasize how well agents
\emph{recall} past events.
They provide limited evidence about context-independent transfer from past
experience.
A system that accumulates a better-organized filing cabinet can score
near-perfectly on recall benchmarks while exhibiting zero genuine
learning---and the field has no way to detect this.

\textbf{The field is optimizing the wrong metric.}
Benchmark designers have focused on recall quality, context window
utilization, and retrieval accuracy.
These measure \emph{lookup quality}, not \emph{learning quality}.
The result is a benchmark ecosystem that rewards better filing cabinets
while saying nothing about whether agents develop expertise.
We propose two redirections.

\textbf{Redirection 1: Compositional Generalization over Time (CGT).}
The single most important metric the field is missing:
\emph{does an agent's ability to handle novel concept combinations
improve with experience?}
Concretely, operate an agent on a domain for $T$ sessions, exposing it
to concepts only in isolation.
After $T$ sessions, evaluate on queries requiring combining two or more
concepts in ways not seen during operation.
Compare four arms built on the same base model: episodic retrieval,
retrieved-rule summaries, consolidation, and a hybrid.
Report both normal memory-assisted accuracy and transfer when the specific
acquired trace or rule is withheld, while matching total examples, training and
inference compute, tokens, latency, and stored bytes.
The performance hypothesis is falsified in a domain if a fixed-state contextual
system consistently matches an updateable system on held-out compositions
under this matched budget.

\textbf{Redirection 2: Beyond SFT evaluation.}
Mechanistic studies suggest that some SFT updates primarily alter access
patterns rather than the targeted fact-memory
units~\citep{yao2024knowledge,ye2025analyzing}.
Benchmarks should therefore distinguish representation change from improved
access instead of treating every SFT gain as durable knowledge acquisition.
Future benchmarks should evaluate knowledge update via targeted weight
editing, circuit-aware fine-tuning, and self-distillation from agent
traces---methods that modify the representations where knowledge actually
resides.

\textbf{What not to do.}
Do not equate ``memory capacity'' with context-window length alone.
Longer context can improve assisted performance, but it does not change
persistent adaptive state and only avoids Theorem~\ref{thm:gap} when it raises
contextual accuracy to the target.
Nor should benchmark designers aggregate raw scores across heterogeneous domains
without calibration: evidence from ML peer review shows that a fixed numerical
scale can become non-comparable across research areas~\citep{xu2026reviewer}.
Prioritize generalization over capacity.

\subsection{For the Continual Learning Community: The Agentic Setting Is Your Deployment Target}

The agentic setting offers a natural deployment opportunity for continual
learning methods~\citep{cossu2022continual}, while exposing their unresolved
stability, scoping, and validation problems.

Agentic systems provide what continual learning methods need:
a natural experience stream with reward labels (task success or
failure), a clear generalization criterion (compositional novelty),
and immediate practical value (deployed agents that should improve
over time).
Continual learning investigates machinery that agentic systems need:
the consolidation machinery (replay, regularization, progressive
expansion) intended to convert episodic experience into durable knowledge
while mitigating, but not eliminating, catastrophic forgetting.

The research agenda that follows is concrete.
\emph{Experience selection}: not all agent traces carry generalizable
signal; the CL community's work on coreset selection and influence
functions is directly applicable to deciding what to consolidate.
\emph{Consolidation scheduling}: the trade-off between consolidation
frequency and stability mirrors the online-vs-batch trade-off in
continual learning; existing scheduling
theory~\citep{kirkpatrick2017ewc} applies.
\emph{Validation}: how to verify that consolidation produced genuine
compositional generalization rather than rote memorization of
distilled examples is an open empirical question that the CL community
is uniquely positioned to answer.
The agentic memory ecosystem is not a competing paradigm; it is the
fast-learning half waiting for its slow-learning complement.

\section{Related Work}
\label{sec:related}

Widely used agent memory architectures emphasize $C$-engineering:
MemGPT~\citep{packer2023memgpt}, Generative Agents~\citep{park2023generative},
Reflexion~\citep{shinn2023reflexion}, and Voyager~\citep{wang2023voyager} treat
retrieval as the default persistence mechanism; more recent systems
(MemoryBank~\citep{zhong2024memorybank}, A-MEM~\citep{xu2025amem},
mem0~\citep{chhikara2025mem0}) continue the same paradigm with richer indexing.
Surveys~\citep{zhang2024survey,hu2025memoryage,wu2026memoryera} catalog this
space. Prompting work demonstrates that contextual systems can perform genuine
compositional induction~\citep{zhou2023leasttomost,drozdov2022compositional,
chen2024skic,abedsoltan2025task}; these results delimit, rather than contradict,
our conditional theorem.
Empirically, \citet{ovadia2024finetuning}, \citet{yang2026multihop}, and
\citet{mosbach2023fewshot} support task-dependent complementarity between
retrieval, ICL, and fine-tuning rather than a universal winner.
\citet{yao2026parammem} provide directional evidence for parametric reflective
memory on held-out transfers, but not a universal compute-matched comparison.
RETRO and kNN-LM~\citep{borgeaud2022retro,khandelwal2019generalization} cover
static corpora but not post-deployment accumulation;
\citet{yu2026selfconsolidate} and MemOS~\citep{li2025memos} realize variants
of consolidation channels.

On security, \citet{greshake2023indirect}, MINJA~\citep{dong2025minja},
PoisonedRAG~\citep{zou2024poisonedrag}, and
InjecAgent~\citep{zhan2024injecagent} document injection attacks against
memory-augmented agents; Proposition~\ref{prop:attack} isolates a conditional
persistent-write path rather than ranking substrate safety.
\citet{zhang2026experience} unify memory, skills, and rules as an experience
compression spectrum; we add a distinction based on which acquired state
persists.
Boundary cases~\citep{latentlearning2025,memrl2026,memento2025} correspond
to the $\bar{\alpha}\to 1$ regime of Assumption~\ref{asm:icl} where
Theorem~\ref{thm:gap}'s separation narrows as predicted;
\citet{lee2026skillSD} independently implement the Trace$\to$Skill$\to$FT
pipeline we propose.
Mechanistic interpretability~\citep{meng2022rome,meng2023memit} and continual
learning methods~\citep{kirkpatrick2017ewc,huang2024ssr,hu2022lora} supply
candidate tools, while their safe frontier-scale integration remains open.

\section{Scope and Limitations}
\label{sec:limitations}

The performance separation is conditional on measured contextual accuracy and
on an appropriate parametric learner for the target class; it is not a
universal ranking of ICL and fine-tuning.
We do not provide a compute-matched frontier-model experiment demonstrating the
Frozen Novice effect, and existing empirical comparisons are task-dependent.
Nor do we demonstrate reliable large-scale consolidation: forgetting, recency
bias, contamination, task and user interference, unlearning, reproducibility,
and certification remain substantial barriers.
The security analysis compares attack paths, not absolute safety.
The performance component of our position would be falsified in any domain
where a fixed-state contextual system consistently matches an updateable system
on held-out post-deployment compositions under the same experience and resource
budget. In such a domain, retrieval may be entirely sufficient even though the
two systems still update different state.

\section{Conclusion}
\label{sec:conclusion}

Agentic memory persists acquired external state. A frozen model can reason
abstractly over that state, but the acquired capability remains contingent on
supplying it.
Under an explicit bounded-contextual-accuracy condition, this creates a
\textbf{Generalization Gap}; across sessions it creates a \textbf{Frozen
State} distinction between assisted performance and context-independent
competence; and, when untrusted content can be written persistently, it creates
a \textbf{security compounding} path.
We therefore advocate pairing scoped episodic retrieval with a separately
governed adaptation channel when durable transfer warrants its cost and risk.
These systems are becoming better filing systems---often highly useful
ones---while durable internalization remains a different operation.

The position we have staked has immediate methodological consequences.
System builders should investigate a governed adaptation channel from agent
experience to persistent learned state.
Benchmark designers should measure \textbf{Compositional Generalization over
Time}: does performance on novel concept combinations improve with experience?
The continual learning community should re-engage the agentic setting, which
provides the experience stream and success criterion their methods require.
The experience compression spectrum~\citep{zhang2026experience} helps make the
allocation explicit: scoped and reversible information belongs naturally in
$C$, while context-independent behavioral adaptation requires a persistent
learning substrate.

\bibliographystyle{plainnat}
\bibliography{references}

\appendix

\section{A Stylized Context-Capacity Example}
\label{app:bound}

\textbf{Setup.}
Let $\mathcal{T}_m$ be the class of $m$-hop chain reasoning tasks: given
query $q = (e_0, r_1, \ldots, r_m)$ and fact base
$\mathcal{F} = \{(e_{i-1}, r_i, e_i)\}_{i=1}^m$, the agent must return $e_m$.

\textbf{Assumptions.}
Suppose each relation is task-instance-specific, cannot be reconstructed from
the other relations or pretrained knowledge, and must be available for a
correct answer. Suppose further that retrieval inserts at most $K$ atomic
relations into context.
Then for $m>K$, at least one necessary relation is absent and this restricted
retrieval system cannot guarantee a correct answer.
\citet{liu2023lostmiddle} document positional degradation in long contexts,
while \citet{paulsen2025effectivecontext} motivate measuring effective rather
than nominal context.

\textbf{Interpretation.}
An adaptive system that has correctly encoded the required relation structure
need not retrieve every atomic fact, but this example does not prove that
fine-tuning will learn that structure.
Larger windows, lossless compression, iterative retrieval, or a rule supplied
in context can also remove the gap. The example is therefore a capacity
boundary case, not a universal substrate ordering.

\section{Proof of Theorem~\ref{thm:gap} (Compositional Sample Complexity Separation)}
\label{app:gap_proof}

\textbf{Full setup.}
Let $\mathcal{F} = \{f_1, \ldots, f_k\}$ be the set of base concepts in
a target domain, and $\oplus: \mathcal{F} \times \mathcal{F} \to \mathcal{Y}$
the domain-specific composition operator.
Let $N = \binom{k}{2}$ denote the number of distinct unordered concept pairs.
Both systems observe the same training data
$D = \{(f_{i_t}, f_{j_t}, \oplus(f_{i_t}, f_{j_t}))\}_{t=1}^n$,
drawn uniformly from the $N$ possible pairs.
Let $S \subseteq \binom{[k]}{2}$ be the set of pairs covered by $D$,
with $|S| = n \leq N$.

\textbf{Formal definition of $\CGC$.}
\[
\CGC(M, D) \;=\;
\Pr_{(f_i, f_j) \sim \mathrm{Uniform}\bigl(\binom{\mathcal{F}}{2}\bigr)}
\!\bigl[\, M(f_i, f_j) = \oplus(f_i, f_j) \,\bigr].
\]

\textbf{Part 1: Retrieval lower bound.}

For a retrieval-based system $\MR$ with store $D$ and frozen model
$f_{\theta_{\mathrm{frozen}}}$, consider a query pair $(f_i, f_j)$:

\begin{itemize}[nosep]
\item If $(i,j) \in S$: the exact composition is stored and retrievable;
accuracy is~1.
\item If $(i,j) \notin S$: the system retrieves $K$ most relevant entries
from $D$ and the frozen model processes them.
This processing includes the LLM's full inferential capacity: attention
over retrieved examples, chain-of-thought reasoning, and any other
in-context computation.
By Assumption~\ref{asm:icl}, accuracy on this query is at most
$\bar{\alpha} < 1$.
\end{itemize}

Therefore:
\begin{align}
\CGC(\MR, D)
&= \frac{n}{N} \cdot 1 + \frac{N - n}{N} \cdot \alpha(K) \notag \\
&\leq \frac{n + (N - n)\bar{\alpha}}{N}. \label{eq:ret_bound}
\end{align}

Setting $\CGC(\MR, D) \geq 1 - \delta$ and solving for $n$:
\begin{align}
n + (N-n)\bar{\alpha} &\geq (1 - \delta)N \notag \\
n(1 - \bar{\alpha}) &\geq (1 - \delta - \bar{\alpha})N \notag \\
n \geq n_R &\triangleq \frac{1 - \delta - \bar{\alpha}}{1 - \bar{\alpha}}
\cdot \binom{k}{2}.
\label{eq:nR}
\end{align}

This bound requires $\delta < 1 - \bar{\alpha}$ (otherwise the frozen model
alone achieves the target without any stored compositions).
When $\bar{\alpha}$ is small (the frozen model cannot compose in this domain),
$n_R \approx (1 - \delta) \binom{k}{2} = \Theta(k^2)$.

\textbf{Note on the strength of the retrieval model.}
The bound in~\eqref{eq:ret_bound} is \emph{not} an artefact of modeling
the retrieval system as a na\"ive lookup table.
The term $\alpha(K)$ explicitly accounts for the frozen model's processing
of retrieved content, including multi-step reasoning and in-context
learning over the $K$ retrieved examples.
The only assumption is that this processing cannot perfectly learn an
out-of-distribution composition operator---i.e.,
$\alpha(K) \leq \bar{\alpha} < 1$---which is an empirically testable
claim about the frozen model's ICL capacity on the target domain.

\textbf{Part 2: Parametric upper bound.}

A parametric system $\MP$ fine-tunes on $D$ to learn $\oplus$ as a function
$\hat{h}: \mathcal{F} \times \mathcal{F} \to \mathcal{Y}$ from a realizable
hypothesis class $\mathcal{H}$ with complexity $d$ (VC dimension for binary
outputs; Natarajan dimension for multiclass outputs).

By the fundamental theorem of statistical learning
theory~\citep{vapnik1998statistical}, for any distribution over
$\mathcal{F} \times \mathcal{F}$ and any $\delta,\beta > 0$, if
$\oplus \in \mathcal{H}$ then
\[
n_P = O\!\left(
\frac{d\log(1/\delta)+\log(1/\beta)}{\delta}
\right)
\]
i.i.d.\ examples suffice for an appropriate learner to satisfy
$\Pr[\hat{h}(f_i, f_j) \neq \oplus(f_i, f_j)] \leq \delta$ with
probability at least $1-\beta$, i.e.\
$\CGC(\MP, D) \geq 1 - \delta$.

Crucially, $n_P$ depends only on $d$ (the intrinsic complexity of $\oplus$)
and $\delta$ (the target accuracy), \emph{not} on $k$ (the number of base
concepts).

\textbf{Part 3: Separation.}

For fixed $\delta,\beta$ and a uniform positive contextual margin, the
sample complexity ratio is:
\[
\frac{n_R}{n_P}
= \frac{\frac{1-\delta-\bar{\alpha}}{1-\bar{\alpha}} \binom{k}{2}}
{O\!\left(
\frac{d\log(1/\delta)+\log(1/\beta)}{\delta}
\right)}
= \Omega\!\left(\frac{k^2}{d}\right).
\]

For structured operators with $d = O(k)$: $n_R / n_P = \Omega(k)$.
For operators with $d = O(1)$ (e.g.\ parameterized group operations):
$n_R / n_P = \Omega(k^2)$.

\textbf{Part 4: Dependence on contextual inference.}

Increasing $K$ or improving prompting and retrieval can raise
$\bar{\alpha}$. For any fixed setting with
$\delta < 1-\bar{\alpha}$, the lower bound remains a constant fraction of
$N=\binom{k}{2}$. If contextual inference instead reaches the target
$1-\delta$, the premise fails and the theorem makes no retrieval lower-bound
claim. The result is therefore conditional on measured contextual accuracy,
not an unconditional statement that larger contexts cannot close the
task-level gap.
\hfill$\square$

\section{Proof of Proposition~1 (Information-Theoretic Bound on ICL Accuracy)}
\label{app:fano}
\label{prop:fano}

\textbf{Part 1: Unstructured operators.}

For $\mathcal{H} = \mathcal{Y}^N$ (the set of all functions from concept
pairs to outputs) with a uniform prior, the values
$\{\oplus(f_i, f_j)\}_{(i,j) \in \binom{[k]}{2}}$ are mutually
independent and each uniformly distributed over $\mathcal{Y}$.
Observing $K$ of these values on distinct pairs reveals nothing about the
remaining $N - K$ values (by independence).
The Bayes-optimal prediction accuracy on any unseen pair is therefore
exactly $1/|\mathcal{Y}|$, so $\bar{\alpha} = 1/|\mathcal{Y}|$.
\hfill$\square$

\textbf{Part 2: Structured operators (corrected Fano reduction).}

Let $V$ index an operator $h_V$ drawn uniformly from a finite family
$\mathcal{H}_0\subseteq\mathcal{H}$ of size $M$, and let
$Z=\{(x_t,h_V(x_t))\}_{t=1}^K$ contain $K$ observed pairs.
Because the inputs are sampled independently of $V$ and each label lies in
$\mathcal{Y}$,
\[
I(V;Z)\leq K\log|\mathcal{Y}|.
\]
Fano's inequality~\citep{cover2006elements} therefore gives, for every
decoder $\widehat V(Z)$,
\[
\Pr(\widehat V\neq V)
\;\geq\;
\eta
\;\triangleq\;
1-\frac{K\log|\mathcal{Y}|+\log 2}{\log M}.
\]
This identification bound alone does \emph{not} imply the previously claimed
uniform prediction margin: distinct posterior hypotheses may disagree on
arbitrarily little probability mass.

Let $U$ be the $N-K$ unobserved pairs and
\[
d_U(h,h')=\frac{1}{N-K}\sum_{x\in U}\mathbf{1}\{h(x)\neq h'(x)\}.
\]
For any predictor $\widehat h_Z$, perfect prediction on all of $U$, together
with the observed labels, identifies $V$. Hence
\[
\mathbb{E}\,d_U(\widehat h_Z,h_V)
\;\geq\;\frac{\eta}{N-K},
\]
a valid but potentially vanishing average-case prediction bound.

A non-vanishing margin requires an explicit separation condition.
Suppose $\mathcal{H}_0$ is a $2\rho$-packing under uniform Hamming distance
on all $N$ pairs:
\[
\frac{1}{N}\sum_x\mathbf{1}\{h(x)\neq h'(x)\}\geq 2\rho
\quad\text{for all }h\neq h'.
\]
Removing the $K$ observed pairs leaves separation at least
$2\epsilon_K=(2\rho N-K)/(N-K)$ on $U$, provided $2\rho N>K$.
Nearest-packing decoding succeeds whenever
$d_U(\widehat h_Z,h_V)<\epsilon_K$.
Markov's inequality and the Fano bound then yield
\[
\mathbb{E}\,d_U(\widehat h_Z,h_V)
\;\geq\;\epsilon_K\eta,
\qquad
\bar{\alpha}\leq 1-\epsilon_K\eta.
\]
Thus Appendix~\ref{app:fano} supplies a uniform positive margin only when
both the packing radius and the Fano term remain positive.
Without this prediction-space separation, Assumption~\ref{asm:icl} must be
measured empirically. This correction affects the auxiliary justification,
not Theorem~\ref{thm:gap}, which takes $\bar{\alpha}$ as an explicit premise.
\hfill$\square$

\section{Comparison of Continual Learning Methods}
\label{app:cl_comparison}

\begin{table}[ht]
\caption{Representative continual learning methods suitable for the offline
consolidation pipeline in Section~\ref{sec:cfa}.
All enable rule-based weight updates from agent experience.}
\label{tab:cl_methods}
\centering
\small
\begin{tabular}{lllll}
\toprule
\textbf{Method} & \textbf{Param cost} & \textbf{Data required} & \textbf{Scale} & \textbf{Ref.} \\
\midrule
Full fine-tuning & $O(d)$ & Original data & $\leq 7$B & \\
LoRA & $O(r{\cdot}d)$, $r \ll d$ & Original data & Any & \citet{hu2022lora} \\
EWC & $O(d)$ & Some original & $\leq 3$B & \citet{kirkpatrick2017ewc} \\
SSR & $O(r{\cdot}d)$ & None (self-gen)& Any & \citet{huang2024ssr} \\
ROME / MEMIT & $O(1)$ per fact & Single fact & Any & \citet{meng2022rome} \\
TTT layers & Session-scoped & Test input & Any & \citet{sun2024tttlayers} \\
\bottomrule
\end{tabular}
\end{table}

\section{Constructive Example and Composition-Error Corollary}
\label{app:alt_gap}

\subsection{Modular arithmetic construction}

We exhibit a concrete family of compositional tasks satisfying
Theorem~\ref{thm:gap}, proving the existence of domains where the
$\Omega(k^2/d)$ separation is realized.

\textbf{Setup.}
Let $p$ be a prime and $k \leq p$.
Define base concepts $\mathcal{F} = \{a_1, \ldots, a_k\}$ as $k$ distinct
elements of $\mathbb{Z}_p$.
Fix a domain-specific constant $c \in \mathbb{Z}_p$ (unknown to both
systems) and define the composition operator:
\[
\oplus(a_i, a_j) = (a_i \cdot a_j + c) \bmod p.
\]

\textbf{Hypothesis class.}
Let $\mathcal{H} = \{h_c : (a, b) \mapsto (ab + c) \bmod p \mid
c \in \mathbb{Z}_p\}$.
A single labeled example $(a_i, a_j, y)$ determines
$c = (y - a_i a_j) \bmod p$, so the VC dimension of $\mathcal{H}$ is
$d = 1$.

\textbf{Applying Theorem~\ref{thm:gap}.}
\begin{itemize}[nosep]
\item \textbf{Parametric:} $n_P = O(1/\delta)$ examples suffice to
identify $c$ and generalize to all $\binom{k}{2}$ pairs.
\item \textbf{Retrieval:} For large $p$, frozen LLMs achieve low
accuracy on modular arithmetic ($\bar{\alpha} \ll 1$).
The retrieval system requires
$n_R = \Omega(k^2)$ stored compositions.
\item \textbf{Separation ratio:} $n_R / n_P = \Omega(k^2)$---quadratic
in the number of domain concepts.
\end{itemize}

\textbf{Naturalistic instantiations.}
The modular arithmetic construction is a formal existence proof; the same
structure arises whenever a domain has many base concepts with a
structured composition rule:
\begin{itemize}[nosep]
\item \emph{Clinical:} drugs $\times$ conditions $\to$ interaction outcomes.
The composition rule (pharmacological interaction) is domain-specific
and post-dates most pretraining corpora.
\item \emph{Legal:} precedents $\times$ statutes $\to$ case outcomes.
The composition rule requires domain expertise absent from general
pretraining.
\item \emph{Engineering:} materials $\times$ loads $\to$ failure modes.
The composition rule is governed by domain equations not reliably
encoded in language model weights.
\end{itemize}
In each case, retrieval can surface the individual components (drug
profiles, precedent summaries, material properties), but the
composition rule must still be supplied, induced in context, or learned in
persistent adaptive state. Assumption~\ref{asm:icl} measures the success of
the complete contextual route rather than ruling induction out.

\subsection{Composition-error corollary}

Let $\varepsilon_C=1-\alpha(K)$ denote the measured error of the
\emph{complete} frozen contextual system on held-out compositions, including
retrieval, prompting, and rule induction.

\begin{proposition}[Conditional error-margin formulation]
\label{prop:altgap}
If $\varepsilon_C>\delta$ and an adapted learner satisfies the realizability
and sample conditions of Theorem~\ref{thm:gap}, then with probability at least
$1-\beta$:
\[
\Acc(\MR,\,\mathcal{T}_{\mathrm{novel}})
\;\leq\;1-\varepsilon_C
\;<\;
1-\delta
\;\leq\;\Acc(\MP,\,\mathcal{T}_{\mathrm{novel}}).
\]
\end{proposition}

\textbf{Relationship to Theorem~\ref{thm:gap}.}
This is a restatement of the theorem's premises, not independent evidence for
them. If $\varepsilon_C\leq\delta$, the proposition abstains and contextual
memory is sufficient for the target.


\end{document}